\documentclass{article}
\newif\ifsup\suptrue
\newif\ifnips\nipsfalse
\ifnips
\usepackage[final]{nips_2016}
\else
\usepackage[final]{nips_anon}
\fi
\usepackage{amsmath}
\usepackage{amssymb}
\usepackage{pgfplots}
\usepackage{pgfplotstable}
\usepgfplotslibrary{groupplots}
\usepackage[boxed]{algorithm}
\usepackage{algpseudocode}
\usepackage{dsfont}
\usepackage[bf,font=small,skip=1pt,justification=centering]{caption}
\usepackage{times}
\usepackage[colorlinks,citecolor=blue,urlcolor=blue]{hyperref}
\usepackage{amsthm}
\usepackage[capitalize]{cleveref}
\usepackage{natbib}

\usepackage{todonotes}

\newcommand{\E}{\mathbb E}
\newcommand{\sr}[1]{\stackrel{#1}}
\newcommand{\set}[1]{\left\{#1\right\}}
\newcommand{\ind}[1]{\mathds{1}\!\!\set{#1}}
\newcommand{\argmax}{\operatornamewithlimits{arg\,max}}

\newcommand{\ceil}[1]{\left \lceil {#1} \right\rceil}
\newcommand{\eqn}[1]{\begin{align}#1\end{align}}
\newcommand{\eq}[1]{\begin{align*}#1\end{align*}}
\newcommand{\logp}{\log_{+}\!}

\def\subsubsect#1{\vspace{1ex plus 0.5ex minus 0.5ex}\noindent{\normalsize\textbf{#1.}}}

\renewcommand{\P}[1]{\mathbb{P}\left\{#1\right\}}
\newcommand{\Pp}[1]{\mathbb{P}'\left\{#1\right\}}

\newcommand{\ocucb}{\scalebox{0.8}{OCUCB-$n$}}

\newcommand{\R}{\mathbb R}
\newcommand{\Z}{\mathbb Z}

\let\epsilon\varepsilon

\theoremstyle{plain}
\newtheorem{theorem}{Theorem}

\newtheorem{lemma}[theorem]{Lemma}

\theoremstyle{definition}

\newtheorem{remark}[theorem]{Remark}

\theoremstyle{remark}

\newcommand{\getN}[2]{\pgfplotstablegetelem{2}{[index]#2}\of#1\pgfplotsretval}

\title{Regret Analysis of the Anytime Optimally Confident UCB Algorithm}
\author{Tor Lattimore \\ Department of Computing Science \\ University of Alberta, Canada \\ \tt{tor.lattimore@gmail.com} }
\date{}

\begin{document}

\maketitle

%%%%%%%%%%%%%%%%%%%%%%%%%%%%%%%%%%%%%%%%%%%%%%%%%%%%%%%%%%%%%%%
% ABSTRACT
%%%%%%%%%%%%%%%%%%%%%%%%%%%%%%%%%%%%%%%%%%%%%%%%%%%%%%%%%%%%%%%

\begin{abstract}
I introduce and analyse an anytime version of the Optimally Confident UCB (OCUCB) algorithm designed for minimising the 
cumulative regret in finite-armed stochastic bandits with subgaussian noise. The new algorithm is simple, intuitive (in hindsight) and comes with the 
strongest finite-time regret guarantees for a horizon-free algorithm so far. I also show a finite-time lower bound that nearly matches the upper bound.
\end{abstract}

%%%%%%%%%%%%%%%%%%%%%%%%%%%%%%%%%%%%%%%%%%%%%%%%%%%%%%%%%%%%%%%
% INTRODUCTION
%%%%%%%%%%%%%%%%%%%%%%%%%%%%%%%%%%%%%%%%%%%%%%%%%%%%%%%%%%%%%%%
\section{Introduction}

The purpose of this article is to analyse an anytime version of the Optimally Confident UCB algorithm
for finite-armed subgaussian bandits \citep{Lat15-ucb}. 
For the sake of brevity I will give neither a detailed introduction nor an exhaustive survey of the literature. Readers looking for a gentle primer
on multi-armed bandits might enjoy the monograph by \cite{BC12} from which I borrow notation.
Let $K$ be the number of arms and $I_t \in \set{1,\ldots,K}$
be the arm chosen in round $t$. The reward is $X_t = \mu_{I_t} + \eta_t$ where $\mu \in \R^K$ is the unknown vector of means and the noise term $\eta_t$ 
is assumed to be $1$-subgaussian (therefore zero-mean). The $n$-step pseudo-regret of strategy $\pi$ given mean vector $\mu$ with maximum mean $\mu^* = \max_i \mu_i$ is
\eq{
R^\pi_\mu(n) = n\mu^* - \E \sum_{t=1}^n \mu_{I_t}\,,
}
where the expectation is taken with respect to uncertainty in both the rewards and actions.
In all analysis I make the standard notational assumption that $\mu_1 \geq \mu_2 \geq \ldots \geq \mu_K$.
The new algorithm is called OCUCB-$n$ and depends on two parameters $\eta > 1$ and $\rho \in (1/2,1]$. The algorithm 
chooses $I_t = t$ in rounds $t \leq K$ and subsequently $I_t = \argmax_i \gamma_i(t)$ with
\eqn{
\label{eq:aucb}
\gamma_i(t) = \hat \mu_i(t-1) + \sqrt{\frac{2\eta \log(B_i(t-1))}{T_i(t-1)}}\,,
}
where $T_i(t-1)$ is the number of times arm $i$ has been chosen after round $t-1$ and $\hat \mu_i(t-1)$ is its empirical estimate and
\eq{
B_i(t-1) = \max\set{e,\, \log(t),\, t\log(t) \left(\sum_{j=1}^K \min\set{T_i(t-1),\, T_j(t-1)^\rho T_i(t-1)^{1-\rho}}\right)^{-1}}\,.
}
Besides the algorithm, the contribution of this article is a proof that OCUCB-$n$ satisfies a nearly optimal regret bound.

\begin{theorem}\label{thm:main}
If $\rho \in [1/2, 1]$ and $\eta > 1$, then
\eq{
R^{\text{\scalebox{0.8}{OCUCB-$n$}}}_\mu(n) 
\leq C_{\eta} \sum_{i : \Delta_i > 0} \left(\Delta_i + \frac{1}{\Delta_i} \log\max\set{\frac{n \Delta_i^2 \log(n)}{k_{i,\rho}},\, \log(n)}\right)\,,
}
where $\Delta_i = \mu^* - \mu_i$ and $k_{i,\rho} = \sum_{j=1}^K \min\{1,\, \Delta_i^{2\rho}/\Delta_j^{2\rho}\}$ and $C_\eta > 0$ is a constant that depends 
only on $\eta$.
Furthermore, for all $\rho \in [0,1]$ it holds that $\limsup_{n\to\infty} R_\mu^{\text{\ocucb}}(n) / \log(n) \leq \sum_{i:\Delta_i > 0} \frac{2\eta}{\Delta_i}$.
\end{theorem}

Asymptotically the upper bound matches lower bound given by \cite{LR85} except for a factor of $\eta$. 
In the non-asymptotic regime the additional terms inside the logarithm significantly improves on UCB.
The bound in \cref{thm:main} corresponds to a worst-case regret that is suboptimal by a factor of just $\smash{\sqrt{\log \log n}}$.
Algorithms achieving the minimax rate are MOSS \citep{AB09} and OCUCB, but both require advance knowledge of the horizon.
The quantity $k_{i,\rho} \in [1,K]$ may be interpreted as the number of ``effective'' arms with larger values leading to improved regret. A simple observation is
that $k_{i,\rho}$ is always non-increasing in $\rho$, which makes $\rho=1/2$ the canonical choice. 
In the special case that all suboptimal arms have the same expected payoff, then $k_{i,\rho} = K$ for all $\rho$.
Interestingly I could not find a regime for which the algorithm is empirically sensitive 
to $\rho \in [1/2, 1]$. 
If $\rho = 1$, then except for $\log\log$ additive terms the problem dependent regret enjoyed by OCUCB-$n$ is equivalent to OCUCB.
Finally, if $\rho = 0$, then the asymptotic result above applies, but
the algorithm in that case essentially reduces to MOSS, which is known to suffer suboptimal finite-time regret in certain regimes \citep{Lat15-ucb}.

\subsubsect{Intuition for regret bound} 
Let us fix a strategy $\pi$ and mean vector $\mu \in \R^K$ and suboptimal arm $i$.
Suppose that $\E[T_i(n)] \leq \Delta_i^{-2} \log(1/\delta) / 2$ for some $\delta \in (0,1)$.
Now consider the alternative mean reward $\mu'$ with $\mu'_j = \mu_j$ for $j \neq i$ and $\mu_i' = \mu_i + 2\Delta_i$, which means that
$i$ is the optimal action for mean vector $\smash{\mu'}$. Standard information-theoretic analysis shows that $\mu$ and $\smash{\mu'}$ are not statistically
separable at confidence level $\delta$ and in particular, if $\Delta_i$ is large enough, then $\smash{R^\pi_{\mu'}(n) = \Omega(n \delta \Delta_i)}$.
For mean $\mu'$ we have $\Delta_j' = \mu_i' - \mu_j' \approx \max\{\Delta_i, \Delta_j\}$ and for any reasonable algorithm we would like
\eq{
\sum_{j:\Delta'_j > 0} \frac{\log(n)}{\Delta_j'} \geq R^\pi_{\mu'}(n) = \Omega(n\delta \Delta_i)\,.
}
But this implies that $\delta$ should be chosen such that
\eq{
\delta = O\left(\frac{\log(n)}{n} \sum_{j:\Delta'_j > 0} \frac{1}{\Delta'_j \Delta_i}\right)
%= O\left(\frac{\log(n)}{n} \sum_{j:\Delta_j > 0} \min\set{\frac{1}{\Delta_i^2},\, \frac{1}{\Delta_i \Delta_j}}\right)
= O\left(\frac{\log(n)k_{i,1/2}}{n \Delta_i^2}\right)\,,
}
which up to $\log \log$ terms justifies the near-optimality of the regret guarantee given in \cref{thm:main} for $\rho$ close to $1/2$.
Of course $\Delta$ is not known in advance, so no algorithm can choose this confidence level.
The trick is to notice that arms $j$ with $\Delta_j \leq \Delta_i$ should be played about as often as arm $i$ and arms $j$ with $\Delta_j > \Delta_i$
should be played about as much as arm $i$ until $\smash{T_j(t-1) \approx \Delta_j^{-2}}$. This means that as $T_i(t-1)$ approaches the critical
number of samples $\smash{\Delta_i^{-2}}$ we can approximate
\eq{
\sum_{j=1}^K \min\set{T_i(t-1),\, T_j(t-1)^{\frac{1}{2}} T_i(t-1)^{\frac{1}{2}}} \approx
\sum_{j=1}^K \min\set{\Delta_i^{-2},\, \Delta_j^{-1} \Delta_i^{-1}} 
= \frac{k_{i,1/2}}{\Delta_i^2}\,.
}
Then the index used by OCUCB-$n$ is justified by 
ignoring $\log\log$ terms and the usual $n \approx t$ used by UCB
and other algorithms. Theorem \ref{thm:main} is proven by making the above approximation rigorous.
The argument for this choice of confidence level is made concrete 
\ifsup
in Appendix \ref{app:lower}
\else
in the supplementary material
\fi
where I present a lower bound that matches the upper bound except for $\log\log(n)$ additive terms.

%%%%%%%%%%%%%%%%%%%%%%%%%%%%%%%%%%%%
% CONCENTRATION
%%%%%%%%%%%%%%%%%%%%%%%%%%%%%%%%%%%%
\section{Concentration}\label{sec:conc}

The regret guarantees rely on a number of concentration inequalities.
For this section only let $X_1,X_2,\ldots$ be i.i.d.\ 1-subgaussian and $S_n = \sum_{t=1}^n X_t$ and $\hat \mu_n = S_n / n$.
The first lemma below is well known and follows trivially from the maximal inequality
and the fact that the rewards are $1$-subguassian.

\subsubsect{Important remark}
For brevity I use $O_\eta(1)$ to indicate a constant that depends on $\eta$ but not other variables such as $n$ and $\mu$. 
The dependence is never worse than polynomial in $1/(\eta - 1)$.

\begin{lemma}\label{lem:conc}
If $\epsilon > 0$, then
$\displaystyle \P{\exists t \leq n : S_t \geq \epsilon} \leq \exp\left(-\frac{\epsilon^2}{2n}\right)$.
\end{lemma}

The following lemma analyses the likelihood that $S_n$ ever exceeds $f(n) = \sqrt{2 \eta n \log \log n}$ where $\eta > 1$.
By the law of the iterated logarithm  $\smash{\limsup_{n\to\infty} S_n / f(n) = \sqrt{1/\eta}}$ a.s.\ and for small $\delta$
it has been shown by \cite{Gar13} that
\eq{
\P{\exists n : S_n \geq \sqrt{2n \log \left(\frac{\log(n)}{\delta}\right)}} = O(\delta)\,.
}
The case where $\delta = \Omega(1)$ seems not to have been analysed and relies on the usual peeling trick, but without
the union bound.

\begin{lemma}\label{lem:lil}
There exists a monotone non-decreasing function $p:(1,\infty) \to (0,1]$ such that for all $\eta > 1$ it holds that
$\P{\forall n : S_n \leq \sqrt{2\eta n \log \max\set{e,\, \log n}}} \geq p(\eta)$.
\end{lemma}

\begin{lemma}\label{lem:tail}
Let $b > 1$ and $\Delta > 0$ and $\tau = \min\set{n : \sup_{t \geq n} \hat \mu_t + \sqrt{\frac{2\eta \log(b)}{t}} < \Delta}$, then
\eq{
\E[\tau] \leq \sqrt{\E[\tau^2]} = O_\eta(1) \cdot \left(1 + \frac{1}{\Delta^2} \logp(b)\right) \qquad \text{where } \logp(x) = \max\set{1, \log(x)}\,.
}
\end{lemma}

The final concentration lemma is quite powerful and forms the lynch-pin of the following analysis.

\begin{lemma}\label{lem:conc2}
Let $\Delta > 0$ and $\rho \in [0,1]$ and $d \in \set{1,2,\ldots}$ and $\lambda_1,\ldots,\lambda_d \in [1, \infty]$ be constants. 
Furthermore, let $\alpha$ be the random variable given by
\eq{
\alpha = \inf\set{\alpha \geq 0 : \inf_s \hat \mu_s + \sqrt{\frac{2\eta}{s} \log\max\set{1,\, \frac{\alpha}{\sum_{i=1}^d \min\set{s,\, \lambda_i^\rho s^{1-\rho}}} }} \geq -\Delta}\,.
}
Finally let $\beta = \inf\set{\beta \geq 0 : \beta \log(\beta) = \alpha}$. Then
\begin{enumerate}
\item[(a)] If $\rho \in (1/2,1]$, then $\displaystyle \Delta\E[\alpha] = O\left(\frac{1}{(2\rho-1)(\eta - 1)^2}\right) \cdot \sum_{i=1}^d \min\set{\Delta^{-1},\, \sqrt{\lambda_i}}$ 
\item[(b)] If $\rho \in [1/2,1]$, then $\displaystyle \Delta\E[\beta] = O\left(\frac{1}{(\eta - 1)^2}\right) \cdot \sum_{i=1}^d \min\set{\Delta^{-1},\, \sqrt{\lambda_i}}$
\end{enumerate}
\end{lemma}

The proofs of \cref{lem:lil,lem:tail,lem:conc2} may be found in \ifsup \cref{app:lem:lil,app:lem:tail,app:lem:conc2}. \else the supplementary material. \fi
%%%%%%%%%%%%%%%%%%%%%%%%%%%%%%%%%%%%%%%%%%%%%%%%%%%%%%%%%%%%%%%
% PROOF OF WARMUP
%%%%%%%%%%%%%%%%%%%%%%%%%%%%%%%%%%%%%%%%%%%%%%%%%%%%%%%%%%%%%%%
\section{Analysis of the KL-UCB+ Algorithm}\label{sec:warmup}

Let us warm up by analysing a simpler algorithm, which chooses the arm that maximises the following index.
\eqn{
\label{eq:ucb+}
\gamma_i(t) = \hat \mu_i(t-1) + \sqrt{\frac{2\eta}{T_i(t-1)} \log\left(\frac{t}{T_i(t-1)}\right)}\,.
}
Strategies similar to this have been called KL-UCB+ and suggested as a heuristic by \cite{GC11} 
(this version is specified to the subgaussian noise model). 
Recently \cite{Kau16} has established the asymptotic optimality of strategies with approximately this form, but finite-time analysis has not been
available until now.
Bounding the regret will follow the standard path of bounding $\E[T_i(n)]$ for each suboptimal arm $i$. 
Let $\hat \mu_{i,s}$ be the empirical estimate of the mean of the $i$th arm having observed $s$ samples.
Define $\tau_i$ and $\tau_\Delta$ by
\eq{
\tau_i &= \min\set{t \geq 1/\Delta_i^2 : \sup_{s \geq t} \hat \mu_{i,s} + \sqrt{\frac{2\eta}{s} \log\left(n\Delta_i^2\right)} < \mu_i + \frac{\Delta_i}{2}} \\
\tau_\Delta &= \min\set{t : \inf_{1 \leq s \leq n} \hat \mu_{1,s} + \sqrt{\frac{2\eta}{s} \log\max\set{1,\, \frac{t}{s}}} \geq \mu_1 - \frac{\Delta_i}{2}}\,.
}
If $T_i(t-1) \geq \tau_i$ and $t \geq \tau_{\Delta_i}$, then by the definition of $\tau_{\Delta_i}$ we have $\gamma_1(t) \geq \mu_i + \Delta_i / 2$ and by the 
definition of $\tau_i$ 
\eq{
\gamma_i(t) 
= \hat \mu_i(t-1) + \sqrt{\frac{2\eta \log(t/T_i(t-1))}{T_i(t-1)}}
\leq \hat \mu_i(t-1) + \sqrt{\frac{2\eta \log(n\Delta_i^2)}{T_i(t-1)}} < \mu_i + \frac{\Delta_i}{2}\,,
}
which means that $I_t \neq i$. Therefore $T_i(n)$ may be bounded in terms of $\tau_i$ and $\tau_{\Delta_i}$ as follows:
\eq{
T_i(n) 
&= \sum_{t=1}^n \ind{I_t = i}
\leq \tau_{\Delta_i} + \sum_{t=\tau_{\Delta_i}+1}^n \ind{I_t = i \text{ and } T_i(t-1) < \tau_i} 
\leq \tau_i +  \tau_{\Delta_i}\,. 
}
It remains to bound the expectations of $\tau_i$ and $\tau_{\Delta_i}$.
By Lemma \ref{lem:conc2}a with $d = 1$ and $\rho = 1$ and $\lambda_1 = \infty$ it follows that $\E[\tau_{\Delta_i}] = O_\eta(1) \cdot \Delta_i^{-2}$ and 
by Lemma \ref{lem:tail} 
\eq{
\E[\tau_i] = O_\eta(1) \cdot \left(1 + \frac{1}{\Delta_i^2} \log(n \Delta_i^2)\right)\,.
}
Therefore the strategy in \cref{eq:ucb+} satisfies:
\eq{
R^{\text{\scalebox{0.8}{KL-UCB+}}}_\mu(n) 
= \sum_{i:\Delta_i > 0} \Delta_i \E[T_i(n)] 
= O_\eta(1) \cdot \sum_{i:\Delta_i > 0} \left(\Delta_i + \frac{1}{\Delta_i} \log(n\Delta_i^2)\right)\,.
}
\begin{remark}
Without changing the algorithm and by optimising the constants in the proof it is possible to show that
$\limsup_{n\to\infty} R^{\text{\scalebox{0.8}{KL-UCB+}}}_\mu(n) / \log(n) \leq \sum_{i : \Delta_i > 0} 2\eta / \Delta_i$, which is just a factor of $\eta$ away from
the asymptotic lower bound of \cite{LR85}.
\end{remark}

\iffalse
The same technique can be used to analyse the algorithm for the Bernoulli case where $\mu \in [0,1]^K$ and $X_t \sim \operatorname{Bernoulli}(\mu_{I_t})$.
Then the index becomes
\eq{
\gamma_i(t) = \argmax_i \max\set{\theta \in [0,1] : d(\hat \mu_i(t), \theta) \leq \eta \log\left(\frac{t}{T_i(t-1)}\right)}\,,
}
where $d(p,q) = p \log (p/q) + (1-p) \log (1-p)/(1-q)$ with $d(p, p) = 0$ for all $p$.

\begin{theorem}
There exists a constant $C_\eta > 0$ depending only on $\eta$ such that
\eq{
R^\pi_\mu(n) \leq C_\eta \sum_{i:\Delta_i > 0} \left(\Delta_i + \frac{\Delta_i}{d(\mu_i, \mu^*)} \log(n d(\mu_i, \mu^*))\right)\,.
}
\end{theorem}

The proof is not a central part of this article and is deferred to a longer version. The analysis follows essentially as above, but rather than
analyse the time when $\gamma_1(t) \geq \mu_1 - \Delta_i / 2$ and $\gamma_i(t) < \mu_1 - \Delta_i / 2$ you split unequally and $\epsilon$-close to $\mu_1$.
Of course one must also replace the subgaussian concentration guarantees in the previous section with informational equivalents (eg., as by \cite{Gar13}).
\fi

%%%%%%%%%%%%%%%%%%%%%%%%%%%%%%%%%%%%%%%%%%%%%%%%%%%%%%%%%%%%%%%
% PROOF OF MAIN THEOREM
%%%%%%%%%%%%%%%%%%%%%%%%%%%%%%%%%%%%%%%%%%%%%%%%%%%%%%%%%%%%%%%
\section{Proof of Theorem \ref{thm:main}}\label{sec:thm:main}

The proof follows along similar lines as the warm-up, but each step becomes more challenging, especially controlling $\tau_{\Delta}$.

%%%%%%%%%%%%%%%%%%%%%%%%%%%%%%%%%%%%%%%%%%%%%%%%%%%%%%%%%%%%%%%
\subsection*{Step 1: Setup and preliminary lemmas}

\newcommand{\Opt}{\Phi}

Define $\Opt$ to be the random set of arms for which the empirical estimate never drops below the critical boundary given by the law of iterated logarithm. 
\eqn{
\label{def:O}
\Opt = \set{i > 2 : \hat \mu_{i,s} + \sqrt{\frac{2\eta_1 \log \max\set{e, \log s}}{s}} \geq \mu_i \text{ for all } s}\,,
}
where $\eta_1 = (1 + \eta) / 2$.
By \cref{lem:lil}, $\P{i \in \Opt} \geq p(\eta_1) > 0$. It will be important that $\Opt$ 
only includes arms $i > 2$ and that the events $i, j \in \Opt$ are independent
for $i \neq j$.
From the definition of the index $\gamma$ and for $i \in \Opt$ it holds that $\gamma_i(t) \geq \mu_i$ for all $t$.
The following lemma shows that the pull counts for optimistic arms ``chase'' those of other arms up the point that they become clearly suboptimal. 

\begin{lemma}\label{lem:chase}
There exists a constant $c_\eta \in (0, 1)$ depending only on $\eta$ such that if
(a) $j \in \Opt$ and (b) $\hat \mu_i(t-1) \leq \mu_i + \Delta_i/2$ and (c) $T_j(t-1) \leq c_\eta\min\{T_i(t-1),\, \Delta_j^{-2}\}$, then $I_t \neq i$.
\end{lemma}

\begin{proof}
First note that $T_j(t-1) \leq T_i(t-1)$ implies that $B_j(t-1) \geq B_i(t-1)$.  
Comparing the indices:
\eq{
\gamma_i(t) 
&= \hat \mu_i(t-1) + \sqrt{\frac{2\eta \log B_i(t-1)}{T_i(t-1)}} 
\leq \mu_i + \sqrt{\frac{2\eta c_\eta \log B_j(t-1)}{T_j(t-1)}} + \frac{\Delta_i}{2}\,. 
}
On the other hand, by choosing $c_\eta$ small enough and by the definition of $j \in \Opt$:
\eq{
\gamma_j(t) 
&= \hat \mu_j(t-1) + \sqrt{\frac{2\eta \log B_j(t-1)}{T_j(t-1)}} 
\geq \mu_j + \sqrt{\frac{2 \eta c_\eta \log B_j(t-1)}{T_j(t-1)}} + \sqrt{\frac{c_\eta}{T_j(t-1)}} \\ 
&\geq \mu_1 + \sqrt{\frac{2\eta c_\eta \log B_j(t-1)}{T_j(t-1)}} 
> \gamma_i(t)\,,
}
which implies that $I_t \neq i$.
\end{proof}

Let $J = \min \Opt$ be the optimistic arm with the largest return where if $\Phi = \emptyset$ we define $J = K+1$ and $\Delta_J = \max_i \Delta_i$.
By \cref{lem:lil}, $i \in \Opt$ with constant probability, which means that $J$ is sub-exponentially distributed with rate dependent on $\eta$ only.
Define $K_{i,\rho}$ by
\eqn{
\label{eq:def:K}
K_{i,\rho} = 1 + c_\eta \sum_{j \in \Opt, j \neq i} \min\set{1,\, \frac{\Delta_i^{2\rho}}{\Delta_j^{2\rho}}}\,,
}
where $c_\eta$ is as chosen in \cref{lem:chase}. 
Since $\P{i \in \Opt} = \Omega(1)$ we will have $K_{i,\rho} = \Omega(k_{i,\rho})$ with high probability (this will be made formal later). Let
\eqn{
\nonumber b_i &= \max\set{\frac{n \Delta_i^2 \log(n)}{k_{i,\rho}},\, \log(n),\, e} 
\qquad \text{and} \qquad B_i = \max\set{\frac{n \Delta_i^2 \log(n)}{K_{i,\rho}},\, \log(n),\, e} \\
\tau_i &= \min\set{s \geq \frac{1}{\Delta_i^2} : \sup_{s'\geq s} \hat \mu_{i,s'} + \sqrt{\frac{2\eta}{s'} \log(B_i)} \leq \mu_i + \frac{\Delta_i}{2}}\,. 
\label{eq:def:bt}
}
The following lemma essentially follows from \cref{lem:tail} and the fact that $J$ is sub-exponentially distributed.
Care must be taken because $J$ and $\tau_i$ are not independent. The proof is found in \ifsup \cref{app:lem:tau}. \else the supplementary material. \fi

\begin{lemma}\label{lem:tau}
$\displaystyle \E[\tau_i] \leq \E[J\tau_i] = O_\eta(1) \cdot \left(1 + \frac{1}{\Delta_i^2} \log(b_i)\right)$.
\end{lemma}

The last lemma in this section shows that if $T_i(t-1) \geq \tau_i$, then either $i$ is not chosen or the index of the $i$th arm is not too large.

\begin{lemma}\label{lem:tau-bound}
If $T_i(t-1) \geq \tau_i$, then $I_t \neq i$ or $\gamma_i(t) < \mu_i + \Delta_i/2$.
\end{lemma}

\begin{proof}
By the definition of $\tau_i$ we have $\tau_i \geq \Delta_i^{-2}$ and $\hat \mu_i(t-1) \leq \mu_i + \Delta_i/2$.
By Lemma \ref{lem:chase}, if $j \in \Opt$ and $T_j(t-1) \leq c_\eta \min\set{\Delta_i^{-2},\, \Delta_j^{-2}}$, then $I_t \neq i$.
Now suppose that $T_j(t-1) \geq c_\eta\min\set{\Delta_i^{-2},\, \Delta_j^{-2}}$ for all $j \in \Opt$.
Then
\eq{
B_i(t-1) 
&= \max\set{e,\, \log(t),\, t \log(t) \left(\sum_{j=1}^K \min\set{T_i(t-1),\, T_j(t-1)^\rho T_i(t-1)^{1-\rho}}\right)^{-1}} \\
&\leq \max\set{e,\, \log(n),\, \frac{n\Delta_i^2 \log(n)}{K_{i,\rho}}} = B_i\,.
}
Therefore from the definition of $\tau_i$ we have that $\gamma_i(t) < \mu_i + \Delta_i / 2$. 
\end{proof}

%%%%%%%%%%%%%%%%%%%%%%%%%%%%%%%%%%%%%%%%%%%%%%%%%%%%%%%%%%%%%%%
\subsection*{Step 2: Regret decomposition}

By \cref{lem:tau-bound}, if $T_i(n) \geq \tau_i$, then 
$I_t \neq i$ or $\gamma_i(t) < \mu_i + \Delta_i/2$. Now we must show there
exists a $j$ for which $\gamma_j(t) \geq \mu_i + \Delta_i / 2$. 
This is true for arms $i$ with $\Delta_i \geq 2\Delta_J$ since 
by definition $\gamma_J(t) \geq \mu_J \geq \mu_i + \Delta_i/2$ for all $t$. 
For the remaining arms we follow the idea used in \cref{sec:warmup} and define a random time for each $\Delta > 0$. 
\eqn{
\label{eq:def:taudelta}
\tau_\Delta = \min\set{t :\inf_{s\geq t} \sup_j \gamma_j(s) \geq \mu_1 - \frac{\Delta}{2}}\,.
}
Then the regret is decomposed as follows
\eqn{
\label{eq:decompose}
R^{\text{\ocucb}}_\mu(n)
\leq 
\E\left[\sum_{i : \Delta_i > 0} \Delta_i \tau_i + 2\Delta_J \tau_{\Delta_J/4} + \sum_{i : \Delta_i < \Delta_J / 4} \Delta_i \tau_{\Delta_i}\right]\,. 
}
The next step is to show that the first sum is dominant in the above decomposition, which will lead to the result via Lemma \ref{lem:tau} to 
bound $\E[\Delta_i \tau_i]$.

%%%%%%%%%%%%%%%%%%%%%%%%%%%%%%%%%%%%%%%%%%%%%%%%%%%%%%%%%%%%%%%
\subsection*{Step 3: Bounding $\tau_{\Delta}$}

This step is broken into two quite technical parts as summarised in the following lemma.
The proofs of both results are quite similar, but the second is more intricate and 
\ifsup 
is given in \cref{app:lem:second-term}. 
\else 
must be deferred to the supplementary material. 
\fi

\begin{lemma}\label{lem:stopping}
The following hold:
\begin{flalign*}
&\text{(a).}\;\;
\E\left[\Delta_J \tau_{\Delta_J/4}\right] \leq O_\eta(1) \cdot \sum_{i:\Delta_i > 0} \sqrt{1 + \frac{\log(b_i)}{\Delta_i^2}}   && \\
&\text{(b).}\;\;
\E\left[\sum_{i:\Delta_i < \Delta_J / 4} \Delta_i \tau_{\Delta_i}\right] \leq O_\eta(1)\cdot \sum_{i:\Delta_i > 0} \sqrt{1 + \frac{\log(b_i)}{\Delta_i^2}}\,. && 
\end{flalign*}
\end{lemma}

\begin{proof}[Proof of Lemma \ref{lem:stopping}a]
Preparing to use Lemma \ref{lem:conc2},
let $\lambda \in (0, \infty]^K$ be given by $\lambda_i = \tau_i$ for $i$ with $\Delta_i \geq 2\Delta_J$ and $\lambda_i = \infty$ otherwise.
Now define random variable $\alpha$ by
\eq{
\alpha 
= \inf\set{\alpha \geq 0 : \inf_s \hat \mu_{1,s} + \sqrt{\frac{2\eta}{s} \log\max\set{1,\, \frac{\alpha}{\sum_{i=1}^K \min\set{s, \lambda_i^\rho s^{1-\rho}}}}} \geq \mu_1-\frac{\Delta_J}{8}}
}
and $\beta = \min\set{\beta \geq 0 : \beta \log(\beta) = \alpha}$. Then for $t \geq \beta$ and abbreviating $s = T_1(t-1)$ we have
\eq{
\gamma_1(t) 
&= \hat \mu_{1,s} + \sqrt{\frac{2\eta}{s} \log B_1(t-1)} \\
&= \hat \mu_{1,s} + \sqrt{\frac{2\eta}{s} \log\left(\max\set{e,\, \log(t),\, \frac{t \log(t)}{\sum_{i=1}^K \min\set{s, T_i(t-1)^\rho s^{1-\rho}}}}\right)} \\
&\geq \hat \mu_{1,s} + \sqrt{\frac{2\eta}{s} \log\max\set{1,\, \frac{\alpha}{\sum_{i=1}^K \min\set{s, T_i(t-1)^\rho s^{1-\rho}}}}} \\
&\geq \hat \mu_{1,s} + \sqrt{\frac{2\eta}{s} \log\max\set{1,\, \frac{\alpha}{\sum_{i=1}^K \min\set{s, \lambda_i^\rho s^{1-\rho}}}}} 
\geq \mu_1 - \frac{\Delta_J}{8}\,,
}
where the second last inequality follows since for arms with $\Delta_i \geq 2\Delta_J$ we have $T_i(n) \leq \tau_i = \lambda_i$ and for other
arms $\lambda_i = \infty$ by definition. The last inequality follows from the definition of $\alpha$.
Therefore $\tau_{\Delta_J/4} \leq \beta$ and so
$\E[\Delta_J \tau_{\Delta_J/4}] \leq \E[\Delta_J \beta]$, which by Lemma \ref{lem:conc2}b is bounded by 
\eqn{
\E[\Delta_J \beta] 
&= \E[\E[\Delta_J\beta|\lambda]] 
\leq O_\eta(1) \cdot \E\left[\ind{\Delta_J > 0} \sum_{i=1}^d \min\set{\Delta_J^{-1},\, \sqrt{\lambda_i}}\right] \nonumber \\
&\leq O_\eta(1) \cdot\E\left[\sum_{i:\lambda_i = \infty, \Delta_i = 0} \frac{\ind{\Delta_J > 0}}{\Delta_{\min}} +
\sum_{i:\Delta_i > 0} \sqrt{\tau_i}\right] 
\leq O_\eta(1) \cdot\E\left[\sum_{i:\Delta_i > 0} \sqrt{\tau_i}\right]\,, \label{eq:easy-1}
}
where the last line follows since $\E[J] = O_\eta(1)$ and
\eq{
\E\left[\sum_{i:\lambda_i = \infty, \Delta_i = 0} \frac{\ind{\Delta_J > 0}}{\Delta_{\min}}\right]
&\leq \E\left[\frac{J}{\Delta_{\min}}\right]
\leq O_\eta(1) \cdot \frac{1}{\Delta_{\min}}
\leq O_\eta(1) \max\set{i : \sqrt{\E[\tau_i]}}\,.
}
The resulting is completed substituting $\E[\sqrt{\tau_i}] \leq \sqrt{\E[\tau_i]}$ into \cref{eq:easy-1} and applying Lemma \ref{lem:tau}
to show that $\E[\tau_i] \leq O_\eta(1) \cdot \left(1 + \frac{\log(b_i)}{\Delta_i^2}\right)$.
\end{proof}

%%%%%%%%%%%%%%%%%%%%%%%%%%%%%%%%%%%%%%%%%%%%%%%%%%%%%%%%%%%%%%%
\subsection*{Step 4: Putting it together}

By substituting the bounds given in \cref{lem:stopping} into \cref{eq:decompose} and applying Lemma \ref{lem:tau} we obtain
\eq{
R^{\text{\ocucb}}_\mu(n) 
&\leq \sum_{i : \Delta_i > 0} \Delta_i \E[\tau_i] + O_{\eta}(1) \cdot \sum_{i:\Delta_i > 0} \sqrt{1 + \frac{\log(b_i)}{\Delta_i^2}} \\
&\leq O_{\eta}(1) \cdot \sum_{i: \Delta_i > 0} \left(\Delta_i + \frac{1}{\Delta_i} \log\max\set{\frac{n\Delta_i^2 \log(n)}{k_{i,\rho}},\, \log(n),\, e}\right)\,,
}
which completes the proof of the finite-time bound.

\subsubsect{Asymptotic analysis}
Lemma \ref{lem:conc2} makes this straightforward.
Let $\epsilon_n = \min\{\frac{\Delta_{\min}}{2}, \log^{-\frac{1}{4}}(n)\}$ and 
\eq{
\alpha_n &= \min\set{\alpha : \inf_s \hat \mu_{1,s} + \sqrt{\frac{2\eta}{s} \log\left(\frac{\alpha}{Ks}\right)} \geq -\epsilon_n}\,.
}
Then by Lemma \ref{lem:conc2}a with $\rho = 1$ and $\lambda_1,\ldots,\lambda_K = \infty$ we 
have $\sup_n \E[\alpha_n] = O_\eta(1) K \epsilon_n^{-2}$. Then we modify the definition of $\tau$ by
\eq{
\tau_{i,n} = \min\set{s : \sup_{s'\geq s} \hat \mu_{i,s} + \sqrt{\frac{2\eta}{s} \log(n \log(n))} \leq \mu_1 - \epsilon_n}\,,
}
which is chosen such that if $T_i(t-1) \geq \tau_{i,n}$, then $\gamma_i(t) \leq \mu_1 - \epsilon_n$.
Therefore
\eq{
R^{\text{\ocucb}}_\mu(n) 
&\leq \Delta_{\max} \E[\alpha_n] + \sum_{i : \Delta_i > 0} \Delta_i \E[\tau_{i,n}] 
\leq O_\eta(1) \cdot \frac{\Delta_{\max}K}{\epsilon_n^2} + \sum_{i : \Delta_i > 0} \Delta_i \E[\tau_{i,n}]\,.
}
Classical analysis shows that $\limsup_{n\to\infty} \E[\tau_{i,n}] / \log(n) \leq 2\eta \Delta_i^{-2}$ and
$\lim_{n\to\infty} \epsilon_n^{-2} / \log(n) = 0$, which implies the asymptotic claim in Theorem \ref{thm:main}. 
\eq{
\limsup_{n\to\infty} \frac{R_\mu^{\text{\ocucb}}(n)}{\log(n)} \leq \sum_{i:\Delta_i > 0} \frac{2\eta}{\Delta_i}\,.
}
This naive calculation demonstrates a serious weakness of asymptotic results. The $\Delta_{\max} K \epsilon_n^{-2}$ term in the regret will typically dominate
the higher-order terms except when $n$ is outrageously large. A more careful argument (similar to the derivation of the finite-time bound) would lead to
the same asymptotic bound via a nicer finite-time bound, but the details are omitted for readability.
Interestingly the result is not dependent on $\rho$ and so applies also to the MOSS-type algorithm that is recovered by choosing $\rho = 0$.

%%%%%%%%%%%%%%%%%%%%%%%%%%%%%%%%%%%%%%%%%%%%%%%%%%%%%%%%%%%%%%%
% DISCUSSION
%%%%%%%%%%%%%%%%%%%%%%%%%%%%%%%%%%%%%%%%%%%%%%%%%%%%%%%%%%%%%%%
\section{Discussion}\label{sec:disc}

The UCB family has a new member. This one is tuned for subgaussian noise and roughly mimics the
OCUCB algorithm, but without needing advance knowledge of the horizon. The introduction of $k_{i,\rho}$ is a minor
refinement on previous measures of difficulty, with the main advantage being that it is very intuitive.
The resulting algorithm is efficient and close to optimal theoretically.
Of course there are open questions, some of which are detailed below.

\subsubsect{Shrinking the confidence level}
Empirically the algorithm improves significantly when the logarithmic terms in the definition of $B_i(t-1)$ are dropped. There are several arguments
that theoretically justify this decision.
First of all if $\rho > 1/2$, then it is possible to replace the $t \log(t)$ term in the definition of $B_i(t-1)$ with just $t$ and
use part (a) of Lemma \ref{lem:conc2} instead of part (b). The price is that the regret guarantee explodes as $\rho$ tends to $1/2$ (also not observed in practice).
The second improvement is to replace $\log(t)$ in the definition of $B_i(t-1)$ with 
\eq{
\logp\left(t \cdot \left(\sum_{j=1}^K \min\set{T_i(t-1), T_j(t-1)^\rho T_i(t-1)^{1-\rho}}\right)^{-1}\right)\,,
}
which boosts empirical performance and
rough sketches suggest minimax optimality is achieved. I leave details for a longer article.

\subsubsect{Improving analysis and constants}
Despite its simplicity relative to OCUCB, the current analysis is still significantly more involved than for other variants of UCB.
A cleaner proof would obviously be desirable.
In an ideal world we could choose $\eta = 1$ or (slightly worse) allow it to converge to $1$ as $t$ grows, which is the technique 
used in the KL-UCB algorithm \citep[and others]{CGMMS13}. I anticipate this would lead to an asymptotically optimal algorithm.

\subsubsect{Informational confidence bounds}
Speaking of KL-UCB, if the noise model is known more precisely (for example, it is bounded), then it is beneficial to use
confidence bounds based on the KL divergence. Such bounds are available and could be substituted directly to 
improve performance without
loss \cite[and others]{Gar13}. Repeating the above analysis, but exploiting the benefits of tighter confidence intervals would be an interesting (non-trivial) problem due
to the need to exploit the non-symmetric KL divergences. It is worth remarking that confidence bounds based on the KL divergence are also {\it not} tight.
For example, for Gaussian random variables they lead to the right exponential rate, but with the wrong leading factor, which in practice can improve performance
as evidenced by the confidence bounds used by (near) Bayesian algorithms that exactly exploit the noise model (eg., \cite{KCOG12,Lat15gittins,Kau16}). This
is related to the ``missing factor'' in Hoeffding's bound studied by \cite{Tal95}.

\subsubsect{Precise lower bounds}
Perhaps the most important remaining problem for the subgaussian noise model is the question of lower bounds.
Besides the asymptotic results by \cite{LR85} and \cite{BK97} there has been some 
recent progress on finite-time lower bounds, both in the OCUCB paper and a recent article by \cite{GMS16}. Some further progress is made in \cref{app:lower}, but
still there are regimes where the bounds are not very precise.

\appendix

%%%%%%%%%%%%%%%%%%%%%%%%%%%%%%%%%%%%%%%%%%%%%%%%%%%%%
% BIBLIOGRAPHY
%%%%%%%%%%%%%%%%%%%%%%%%%%%%%%%%%%%%%%%%%%%%%%%%%%%%%
\bibliographystyle{plainnat}
\bibliography{all}

\ifsup

\section{Lower Bounds}\label{app:lower}

I now prove a kind of lower bound showing that the form of the regret in Theorem \ref{thm:main} is approximately correct for $\rho$ close
to $1/2$. The result contains a lower order $-\log \log(n)$ term, which for large $n$ dominates the improvements, but is 
meaningful in many regimes.

\begin{theorem}
Assume a standard Gaussian noise model
and let $\pi$ be any strategy and $\mu \in [0,1]^K$ be such that $\frac{n\Delta_i^2}{k_{i,1/2} \log(n)} \geq 1$ for all $i$.
Then one of the following holds:
\begin{enumerate}
\item $\displaystyle R^\pi_\mu(n) \geq \frac{1}{4} \sum_{i:\Delta_i > 0} \frac{1}{\Delta_i} \log\left(\frac{n\Delta_i^2}{k_{i,1/2} \log(n)}\right)$.
\item There exists an $i$ with $\Delta_i > 0$ such that
\eq{
R^\pi_{\mu'}(n) \geq \frac{1}{2} \sum_{i:\Delta_i' > 0} \frac{1}{\Delta_i'} \log\left(\frac{n\Delta_i'^2}{k_{i,1/2}' \log(n)}\right)
}
where $\mu'_i = \mu_i + 2\Delta_i$ and $\mu'_j = \mu_j$ for $j \neq i$ and $\Delta_i'$ and $k_{i,\rho}'$ are defined as $\Delta_i$ and $k_{i,\rho}$ but using
$\mu'$.
\end{enumerate}
\end{theorem}

\begin{proof}
On our way to a contradiction, assume that neither of the items hold.
Let $i$ be a suboptimal arm and $\mu'$ be as in the second item above. 
I write $\mathbb{P}'$ and $\E'$ for expectation when when rewards are sampled from $\mu'$.
Suppose
\eqn{
\label{eq:lower1}
\E[T_i(n)] \leq \frac{1}{4\Delta_i^2} \log\left(\frac{n \Delta_i^2}{k_{i,1/2} \log(n)}\right)\,.
}
Then Lemma 2.6 in the book by \cite{Tsy08} and the same argument as used by \cite{Lat15-ucb} gives
\eq{
\P{T_i(n) \geq n/2} + \Pp{T_i(n) < n/2} \geq \frac{k_{i,1/2} \log(n)}{n \Delta_i^2} \equiv 2\delta\,.
}
By Markov's inequality
\eq{
\P{T_i(n) \geq n/2} \leq \frac{2\E[T_i(n)]}{n} \leq \frac{1}{2n\Delta_i^2} \log\left(\frac{n\Delta_i^2}{k_{i,1/2} \log(n)}\right)
\leq \frac{\log(n)}{2n\Delta_i^2} \leq \delta\,.
}
Therefore $\Pp{T_i(n) < n/2} \geq \delta$, which implies that
\eq{
R_{\mu'}^\pi(n) \geq \frac{\delta n \Delta_i}{2} 
%= \frac{k_{i,1/2} \log(n)}{2\Delta_i}
= \frac{1}{2} \sum_{j=1}^K \min\set{\frac{1}{\Delta_i}, \frac{1}{\Delta_j}} \log(n)
%\geq \frac{1}{2} \sum_{j : \Delta'_j > 0} \frac{\log(n)}{\Delta'_j} 
\geq \frac{1}{2} \sum_{j : \Delta'_j > 0} \frac{1}{\Delta'_j} \left(\frac{n\Delta'_j}{k'_{j,1/2} \log(n)}\right)\,,
}
which is a contradiction. Therefore \cref{eq:lower1} does not hold for all $i$ with $\Delta_i > 0$, but this also
leads immediately to a contradiction, since then
\eq{
R_\mu^\pi(n) 
&= \sum_{i:\Delta_i > 0} \Delta_i \E[T_i(n)] 
\geq \frac{1}{4}\sum_{i:\Delta_i > 0} \frac{1}{\Delta_i} \log\left(\frac{n \Delta_i^2}{k_{i,1/2} \log(n)}\right)\,. \qedhere
}
\end{proof}

%%%%%%%%%%%%%%%%%%%%%%%%%%%%%%%%%%%%%%%%%%%%%%%%%%%%%%%%%%%%%%%
% PROOF OF LIL
%%%%%%%%%%%%%%%%%%%%%%%%%%%%%%%%%%%%%%%%%%%%%%%%%%%%%%%%%%%%%%%
\section{Proof of Lemma \ref{lem:lil}}\label{app:lem:lil}

Monotonicity is obvious.
Let $\epsilon > 0$ be such that $\eta = 1+2\epsilon$ and
and $G_k = [(1+\epsilon)^k,\, (1+\epsilon)^{k+1}]$ and $F_k = \set{\exists n \in G_k : S_n > \sqrt{2\eta n \log\max\set{e, \log n}}}$. Then
\eq{
\P{\forall n : S_n \leq \sqrt{2\eta n \log\max\set{e,  \log n}}} 
= \P{\forall k \geq 0 : \neg F_k} 
= \prod_{k=0}^\infty \P{\neg F_k|\neg F_1,\ldots, \neg F_{k-1}}\,.
}
Now we analyse the failure event $F_k$.
\eq{
\P{F_k|\neg F_1,\ldots, \neg F_{k-1}} 
&\leq \P{F_k} \\
&= \P{\exists n \in G_k : S_n > \sqrt{2\eta n \log\max\set{e, \log n}}} \\
&\leq \exp\left(-\frac{2\eta (1+\epsilon)^k \logp \log (1+\epsilon)^k}{2(1+\epsilon)^{k+1}}\right) \\
&= \left(\frac{1}{k \log(1+\epsilon)}\right)^{1 + \frac{\epsilon}{1+\epsilon}}\,.
}
Since this is vacuous when $k$ is small we need also need a naive bound.
\eq{
\P{\exists n \in G_k : S_n \geq \sqrt{2\eta n \log\max\set{e, \log n}}} \leq \exp\left(-\eta\right) < 1\,.  
}
Combining these completes the results since for sufficiently large $k_0$ (depending only on $\eta$) we have that 
\eq{
p(\eta) \geq \exp\left(-\eta k_0\right) \prod_{k=k_0}^\infty (1 - \P{F_k}) 
\geq \exp\left(-\eta k_0\right) \prod_{k=k_0}^\infty \left(1 - \left(\frac{1}{k \log(1+\epsilon)}\right)^{1 + \frac{\epsilon}{1+\epsilon}}\right) > 0\,.
}

%%%%%%%%%%%%%%%%%%%%%%%%%%%%%%%%%%%%%%%%%%%%%%%%%%%%%%%%%%%%%%%
% PROOF OF TAIL CONCENTRATION BOUND
%%%%%%%%%%%%%%%%%%%%%%%%%%%%%%%%%%%%%%%%%%%%%%%%%%%%%%%%%%%%%%%
\section{Proof of Lemma \ref{lem:tail}}\label{app:lem:tail}
Let $\alpha \geq 1$ be fixed and $t_0 = \ceil{8\eta \logp(b) / \Delta^2}$ and $t_k = t_0 2^k$. 
Then
\eq{
\P{\tau \geq \alpha t_0}
&\leq \P{\exists t \geq \alpha t_0 : \hat \mu_t \geq \Delta / 2} 
\leq \sum_{k=0}^\infty \P{\exists t \leq t_k : S_t \geq \alpha 2^{k-1} t_0 \Delta / 2} \\ 
&\leq \sum_{k=0}^\infty \exp\left(-\frac{\alpha^2 2^{2k-2} t_0^2 \Delta^2}{8 \alpha 2^k t_0}\right) 
\leq \sum_{k=0}^\infty \exp\left(-\frac{\alpha 2^k}{4}\right) 
= O\left(\exp\left(-\alpha/4\right)\right)\,.
}
Therefore $\E\left[\left(\tau / t_0\right)^2\right] = O(1)$ and so the result follows.

%%%%%%%%%%%%%%%%%%%%%%%%%%%%%%%%%%%%%%%%%%%%%%%%%%%%%%%%%%%%%%%
% PROOF OF CONCENTRATION BOUND
%%%%%%%%%%%%%%%%%%%%%%%%%%%%%%%%%%%%%%%%%%%%%%%%%%%%%%%%%%%%%%%
\section{Proof of Lemma \ref{lem:conc2}}\label{app:lem:conc2}

Let $\eta_1 = (1 + \eta) / 2$ and $\eta_2 = \eta / \eta_1$ and
\eq{
\Lambda = \sum_{i=1}^d \min\set{\frac{1}{\Delta^2},\, \frac{\lambda_i^\rho}{\Delta^{2-2\rho}} \logp\left(\frac{1}{\lambda_i \Delta^2}\right)}\,.
}
Let $x > 0$ be fixed and let $G_k = [\eta_1^k, \eta_1^{k+1}]$.
We will use the peeling trick. First, by Lemma \ref{lem:conc}.
\eq{
&q_k
= \P{\inf_{s \in G_k} \hat \mu_s + \sqrt{\frac{2\eta}{s} \log\max\set{1,\, \frac{x\Lambda}{\sum_{i=1}^d \min\set{s,\, \lambda_i^{\rho} s^{1-\rho}}}}} \leq -\Delta} \\
&\leq \! \P{\exists s \leq \eta_1^{k+1}\! : S_s +\! \sqrt{2\eta \eta_1^k \log\max\set{1,\,\frac{x\Lambda}{\sum_{i=1}^d \min\set{\eta_1^{k+1}, \lambda_i^\rho \eta_1^{(k+1)(1-\rho)}}}}} + \Delta\eta_1^k \leq 0} \\ 
&\sr{(a)}\leq\! \left(\frac{\sum_{i=1}^d \min\set{\eta_1^{k+1}, \lambda_i^\rho \eta_1^{(k+1)(1-\rho)}}}{x\Lambda}\right)^{\eta_2} \exp\left(-\frac{\Delta^2 \eta_1^{k - 1}}{2}\right) \\
&=\! \left(\frac{\sum_{i=1}^d \min\set{\eta_1^{k+1}, \lambda_i^\rho \eta_1^{(k+1)(1-\rho)}}}{x\Lambda} \exp\left(-\frac{\Delta^2 \eta_1^k}{2\eta}\right)\right)^{\eta_2}\,,
}
where (a) follows by \cref{lem:conc}.
By the union bound 
\eq{
&\P{\inf_s \hat \mu_s + \sqrt{\frac{2\eta}{s} \log\max\set{1,\, \frac{x \Lambda}{\sum_{i=1}^d \min\set{s^\rho,\, \lambda_i s^{1-\rho}}}}} \leq - \Delta} 
\leq \sum_{k=0}^\infty q_k \\
&\qquad\leq\sum_{k=0}^\infty \left(\frac{\sum_{i=1}^d \min\set{\eta_1^{k+1},\,\lambda_i^\rho \eta_1^{(k+1)(1-\rho)}}}{x\Lambda}\exp\left(-\frac{\Delta^2 \eta_1^k}{2\eta}\right)\right)^{\eta_2} \\
&\qquad\leq \left(\frac{1}{x\Lambda}\sum_{i=1}^d \sum_{k=0}^\infty \min\set{\eta_1^{k+1},\,\lambda_i^\rho \eta_1^{(k+1)(1-\rho)}} \exp\left(-\frac{\Delta^2 \eta_1^k}{2\eta} \right) \right)^{\eta_2} \\
&\qquad = O\left(\frac{\eta}{\eta - 1}\right) \cdot x^{-\eta_2}\,,
}
where the last line follows from \cref{lem:tech1}.
Therefore $\P{\alpha \geq x\Lambda} \leq O\left(\frac{\eta}{\eta - 1}\right) \cdot x^{-\eta_2}$\,.

Now the first part follows easily since
$\E[\alpha] \leq \int^\infty_0 \P{\alpha \geq x \Lambda} = O\left(\frac{\eta}{(\eta - 1)^2}\right) \cdot \Lambda$. Therefore
\eq{
\Delta\E[\alpha] 
&\leq O\left(\frac{\eta}{(\eta - 1)^2}\right) \cdot \sum_{i=1}^d \min\set{\frac{1}{\Delta},\, \lambda_i^\rho \Delta^{2\rho - 1} \logp\left(\frac{1}{\lambda_i \Delta^2}\right)} \\
&\leq O\left(\frac{\eta}{(2\rho - 1)(\eta - 1)^2}\right) \cdot \sum_{i=1}^d \min\set{\frac{1}{\Delta},\, \sqrt{\lambda_i}}\,.
}

\newcommand{\plog}{\operatorname{productlog}}
For the second part let $x_0 = \Lambda / \plog(\Lambda)$ where $\plog$ is the inverse of the function $x \to x \exp(x)$.
\eq{
\E[\beta] 
&\leq \int^\infty_0 \P{\beta \geq x} dx 
\leq x_0 + \int^\infty_{x_0} \P{\alpha \geq \frac{x}{\Lambda} \log(x)} dx \\
&\leq x_0 + O\left(\frac{\eta}{\eta-1}\right) \cdot \int^\infty_{x_0} \left(\frac{\Lambda}{x \log(x)}\right)^{\eta_2} dx \\ 
&\leq x_0 + O\left(\frac{\eta}{\eta-1}\right) \cdot \left(\frac{\Lambda}{\log(x_0)}\right)^{\eta_2} \int^\infty_{x_0} x^{-\eta_2} dx \\
&\leq x_0 + O\left(\frac{\eta}{(\eta-1)^2}\right)\cdot \left(\frac{\Lambda}{\log(x_0)}\right)^{\eta_2} x_0^{1 - \eta_2} 
= O\left(\frac{\eta}{(\eta-1)^2}\right) \cdot \frac{\Lambda}{\plog(\Lambda)}\,.
}
If $\Lambda < e$, then the result is trivial. For $\Lambda \geq e$ we have $\plog(\Lambda) \geq 1$. Then
\eq{
\Delta \E[\beta] 
&\leq O\left(\frac{1}{(\eta-1)^2}\right) \cdot \frac{\Delta \Lambda}{\plog(\Lambda)} \\
&\leq O\left(\frac{1}{(\eta-1)^2}\right) \cdot \sum_{i=1}^d \min\set{\frac{1}{\Delta},\, \frac{\lambda_i^\rho \Delta^{2\rho - 1}}{\plog(\Lambda)} \logp\left(\frac{1}{\lambda_i \Delta^2}\right)}\,.
}
By examining the inner minimum we see that if $\Delta \geq \lambda_i^{-\frac{1}{2}}$, then $1/\Delta \leq \lambda_i^{\frac{1}{2}}$.
If $\Delta < \lambda_i^{-\frac{1}{2}}$, then
\eq{
\min\set{\frac{1}{\Delta},\, \frac{\lambda_i^\rho \Delta^{2\rho - 1}}{\plog(\Lambda)} \logp\left(\frac{1}{\lambda_i \Delta^2}\right)}
&< \frac{\lambda_i^{\frac{1}{2}}}{\max\set{1,\,\plog(\Delta^{-2})}} \logp\left(\frac{1}{\lambda_i \Delta^2}\right) \\
&\leq 2\lambda_i^{\frac{1}{2}}\,.
}
Therefore $\E[\Delta t] \leq O\left(\frac{\eta}{(\eta - 1)^2}\right) \cdot \sum_{i=1}^d \min\set{\Delta^{-1}, \sqrt{\lambda_i}}$ as required.

%%%%%%%%%%%%%%%%%%%%%%%%%%%%%%%%%%%%%%%%%%%%%%%%%%%%%%%%%%%%%%%
% PROOF OF LEMMA TAU
%%%%%%%%%%%%%%%%%%%%%%%%%%%%%%%%%%%%%%%%%%%%%%%%%%%%%%%%%%%%%%%
\section{Proof of Lemma \ref{lem:tau}}\label{app:lem:tau}

Since $J$ is sub-exponentially distributed with rate dependent only on $\eta$ we have $\sqrt{\E[J^2]} = O(1)$. By using \cref{lem:tail} we obtain
\eq{
\sqrt{\E[\tau_i^2]} 
&= \sqrt{\E\!\left[\E\left[\tau_i^2|K_{i,\rho}\right]\right]} \\
&= O_\eta(1) \cdot \sqrt{\E\left[\left(1 + \frac{1}{\Delta_i^2} \log(B_i)\right)^2 \right]}
= O_\eta(1) \cdot \left(1 + \frac{1}{\Delta_i^2} \log\left(b_i\right)\right)\,.
}
The latter inequality follows by noting that $B_i \geq e$ and $(1 + c \log(x))^2$ is concave for $x \geq e$ and $c > 0$.
\eq{
\sqrt{\E\left[\left(1 + \frac{1}{\Delta_i^2} \log(B_i)\right)^2 \right]}
&\leq 1 + \frac{1}{\Delta_i^2} \log(\E[B_i]) \\
&= 1 + \frac{1}{\Delta_i^2} \log\left(\E\left[\max\set{\log(n),\, \frac{n \Delta_i^2\log(n)}{K_{i,\rho}}}\right]\right) \\
&= O_\eta(1) \cdot \left(1 + \frac{1}{\Delta_i^2} \log\left(\max\set{\log(n),\, \frac{n \Delta_i^2 \log(n)}{k_{i,\rho}}}\right)\right)\,, 
}
where the last inequality follows from (a) $K_{i,\rho} \geq 1$ and (b) Azuma's concentration inequality implies
that $\P{K_{i,\rho} \leq c_\eta \rho(\eta) k_{i,\rho} / 2} = O(k_{i,\rho}^{-1})$ as shown in the following appendix.
Finally by Holder's inequality
\eq{
\E[J \tau_i] \leq \sqrt{\E[J^2] \E[\tau_i^2]} 
\leq O_\eta(1) \cdot \left(1 + \frac{1}{\Delta_i^2} \log\left(b_i\right)\right)\,.
}

%%%%%%%%%%%%%%%%%%%%%%%%%%%%%%%%%%%%%%%%%%%%%%%%%%%%%%%%%%%%%%%
% PROOF OF LEMMA TAU
%%%%%%%%%%%%%%%%%%%%%%%%%%%%%%%%%%%%%%%%%%%%%%%%%%%%%%%%%%%%%%%
\section{Tail Bound on \texorpdfstring{$K_{i,\rho}$}{K\_(i,p)}}\label{app:lem:K}

Recall that $K_{i,\rho} = 1 + c_\eta \sum_{j \in \Opt, j \neq i} \min\set{1,\,\Delta_i^{2\rho}/\Delta_j^{2\rho}}$ and $k_{i,\rho} = 1 + \sum_{j \neq i}^K \min\set{1, \Delta_i^{2\rho}/\Delta_j^{2\rho}}$.
Therefore by Azuma's inequality and naive simplification we have
\eq{
\P{K_{i,\rho} \leq c_\eta \rho(\eta) k_{i,\rho} / 2}
&\leq\P{\sum_{j \in \Opt, j \neq i} \min\set{1, \Delta_i^{2\rho} / \Delta_j^{2\rho}} \leq \frac{\rho(\eta)}{2} \sum_{j \neq i} \min\set{1, \Delta_i^{2\rho}/\Delta_j^{2\rho}}} \\
&\sr{(a)}\leq \exp\left(-\frac{\left(\rho(\eta) \sum_{j \neq i} \min\set{1, \Delta_i^{2\rho} / \Delta_J^{2\rho}}\right)^2}{2\sum_{j \neq i} \min\set{1, \Delta_i^{2\rho}/\Delta_j^{2\rho}}^2}\right)\\ 
&\sr{(b)}\leq \exp\left(-\frac{\rho(\eta)^2 \sum_{j \neq i} \min\set{1, \Delta_i^{2\rho} / \Delta_J^{2\rho}}}{2}\right) \\
&\sr{(c)}= O(k_{i,\rho}^{-1})\,,
}
where (a) follows from Azuma's inequality and (b) since $\min\{1,x\}^2 \leq \min\{1,x\}$ and (c) by $\exp(-x) \leq 1/x$ for all $x \geq 0$.

%%%%%%%%%%%%%%%%%%%%%%%%%%%%%%%%%%%%%%%%%%%%%%%%%%%%%%%%%%%%%%%
% PROOF OF LEMMA FOR THIRD TERM IN DECOMPOSITION
%%%%%%%%%%%%%%%%%%%%%%%%%%%%%%%%%%%%%%%%%%%%%%%%%%%%%%%%%%%%%%%
\section{Proof of Lemma \ref{lem:stopping}b}\label{app:lem:second-term}

Recall that we are trying to show that
\eqn{
\label{eq:objs}
\E\left[\sum_{i:\Delta_i < \Delta_J / 4} \Delta_i \tau_{\Delta_i}\right] = O\left(\sum_{i:\Delta_i > 0} \Delta_i \E[J\tau_i]\right)\,.
}
Let $E$ be the event that $\Delta_2 \leq \Delta_J / 4$ and define
random sets
$A_1 = \set{i : \Delta_i \in (2\Delta_J, \infty)}$ and 
$A_2 = \set{i : \Delta_i \in [\Delta_J, 2\Delta_J]}$.
For $i \in A_1$ we have $\Delta_i > 2\Delta_J$ and since $J \in \Opt$ we have $\gamma_J(t) \geq \mu_J \geq \mu_1 - \Delta_i / 2$.
Therefore $i \in A_1$  implies that $\tau_{\Delta_i} = 1$ and so $T_i(n) \leq \tau_i$.
Let $\lambda \in (0, \infty]^K$ be given by $\lambda_i = \tau_i$ for $i \in A_1$ and $\lambda_i = \infty$ otherwise. 
\eq{
\alpha = \min\set{\alpha : \inf_s \hat \mu_{2,s} + \sqrt{\frac{2\eta}{s} \log\max\set{1,\, \frac{\alpha}{\sum_{i=1}^K \min\set{s, \lambda_i^\rho s^{1-\rho}}}}}
\geq \mu_2 - \frac{\Delta_J}{4}}\,. 
}
It is important to note that we have used $\hat \mu_{2,s}$ in the definition of $\alpha$ and not $\hat \mu_{1,s}$ that appeared in the proof of part (a)
of this lemma.
The reason is to preserve independence when samples from the first arm are used later.
Let $\beta = \min\set{\beta \geq 0 : \beta \log(\beta) = \alpha}$.
If $E$ holds, then for $t \geq \beta$ we have $\gamma_2(t) \geq \mu_2 - \Delta_J/4 \geq \mu_1 - \Delta_J/2$, which implies that 
\eq{
\ind{E} \sum_{i \in A_2} T_i(n) \leq t_{\Delta_J} + \sum_{i \in A_2} \tau_i \leq \beta + \sum_{i \in A_2} \tau_i\,.
}
Therefore for any $s, t \leq n$ the concavity of $\min\set{s, \cdot}$ and $x \to x^\rho$ combined with Jensen's inequality implies that
\eq{
\ind{E} \sum_{i \in A_2} \min\set{s, T_i(t-1)^\rho s^{1-\rho}} \leq 
\sum_{i \in A_2} \min\set{s, \left(\frac{\beta + \sum_{i \in A_2} \tau_i}{|A_2|}\right)^\rho s^{1-\rho}}\,.
}
We are getting close to an application of Lemma \ref{lem:conc2}.
Let $\omega \in (0, \infty]^K$ be given by
\eq{
\omega_j = \begin{cases}
\tau_j & \text{if } j \in A_1 \\
\beta/|A_2| + \sum_{j \in A_2} \tau_j / |A_2| & \text{if } j \in A_2 \\
\infty & \text{otherwise}\,,
\end{cases}
}
which has been chosen such that for $T_1(t-1) = s$ and if $E$ holds, then
\eqn{
B_1(t-1) 
&\geq \max\set{1,\, \frac{t \log(t)}{\sum_{j=1}^K \min\set{s, T_j(t-1)^\rho s^{1-\rho}}}}  \nonumber \\
&\geq \max\set{1,\, \frac{t \log(t)}{\sum_{j=1}^K \min\set{s, \omega_j^\rho s^{1-\rho}}}}\,.
\label{eq:B-big}
}
Now let $i$ be the index of some arm for which $\Delta_i < \Delta_J / 4$ and define
\eq{
\alpha_i = \min\set{\alpha : \inf_s \hat \mu_{1,s} + \sqrt{\frac{2\eta}{s} \log\max\set{1, \frac{\alpha}{\sum_{j=1}^K \min\set{s, \omega_j^\rho s^{1-\rho}}}}}
\geq \mu_1 - \frac{\Delta_i}{2}}
}
and $\beta_i = \min\set{\beta \geq 0 : \beta \log(\beta) = \alpha_i}$.
Therefore by \cref{eq:B-big}, 
if $E$ holds and $t \geq \beta_i$, then $\gamma_1(t) \geq \mu_1 - \Delta_i / 2$ and so $t_{\Delta_i} \leq \beta_i$.
At last we are able to write $t_{\Delta_i}$ in terms of something for which the expectation can be controlled.
\eqn{
\E\left[\sum_{i : \Delta_i < \Delta_J / 4} \Delta_i \tau_{\Delta_i}\right]
&\leq \E\left[\sum_{i : \Delta_i < \Delta_J / 4} \Delta_i \beta_i\right] \nonumber \\
&\leq O_\eta(1) \cdot \E\left[\sum_{i : \Delta_i < \Delta_J / 4} \sum_{j=1}^K \min\set{\frac{1}{\Delta_i},\, \sqrt{\omega_j}}\right] \nonumber \\
&\leq O_\eta(1) \cdot \E\left[\sum_{i : \Delta_i < \Delta_J / 4} \left(\sum_{j \in A_1} \sqrt{\tau_j} + |A_2| \sqrt{\omega_j} + \frac{J}{\Delta_{\min}}\right)\right] \nonumber \\
&\leq O_\eta(1) \cdot \E\left[\sum_{j \in A_1} J \sqrt{\tau_j} + \frac{J^2}{\Delta_{\min}} + J |A_2| \sqrt{\omega_j}\right]\,. \label{eq:hard-1}
}
The first two terms are easily bounded as we shall soon see. For the last we have 
\eqn{
\E\left[J |A_2| \sqrt{\omega_j}\right] 
&\leq O_\eta(1) \cdot \sqrt{\E\left[|A_2|^2 \omega_j\right]} 
= O_\eta(1) \cdot \sqrt{\E\left[|A_2| \sum_{j \in A_2} \tau_j + |A_2| \beta\right]} \nonumber \\
&\leq O_\eta(1) \cdot \left(\sqrt{\E\left[|A_2| \sum_{j \in A_2} \tau_j\right]} + \sqrt{\E\left[|A_2| \beta\right]}\right) \label{eq:hard-2}
}
Bounding each term separately. For the first, let $\tilde A_\ell = \set{j : \Delta_j \in [2^\ell, 2^{\ell+2})}$, which
is chosen such that no matter the value of $\Delta_J$ there exists an $\ell \in \Z$ with $A_2 \subseteq A_\ell$.
\eqn{
\sqrt{\E\left[|A_2| \sum_{j \in A_2} \tau_j\right]}
&\leq O(1) \cdot \sqrt{\sum_{\ell\in \Z} |\tilde A_\ell| \sum_{j \in \tilde A_\ell} \E[\tau_j]} \nonumber \\ 
&\leq O(1) \cdot \sqrt{\sum_{\ell\in \Z} |\tilde A_\ell|^2\max_{j \in A_{\ell}} \E[\tau_j]} \nonumber \\
&\leq O_\eta(1) \cdot \sum_{j : \Delta_j > 0} \sqrt{1 + \frac{\log(b_j)}{\Delta_j^2}}\,, \label{eq:hard-3}
}
where the last inequality follows because $\sum_{\ell\in\Z} \mathds{1}\{j \in \tilde A_\ell\} = 2$ for each $j$ and from Lemma \ref{lem:tau}, which 
gives the same order-bound on $\E[\tau_j]$ for all $j \in \tilde A_\ell$ for fixed $\ell$.
For the second term in \cref{eq:hard-2} we have
\eq{
\E\left[\sqrt{|A_2| \beta}\right] 
&\sr{(a)}\leq O_\eta(1) \cdot \E\left[\sqrt{|A_2|\left(\sum_{j : \lambda_j = \infty} \frac{1}{\Delta_J^2} + \frac{1}{\Delta_J} \sum_{j : \lambda_j < \infty} \sqrt{\tau_j}\right)}\right] \\
&\sr{(b)}\leq O_\eta(1) \cdot \left(\sqrt{\E\left[|A_2|\sum_{j : \lambda_j = \infty} \frac{1}{\Delta_J^2}\right]} 
  + \E\left[\frac{|A_2|}{\Delta_J}\right] 
  + \sum_{j : \Delta_j > 0} \E[\sqrt{\tau_j}]\right) \\
&\sr{(c)}\leq O_\eta(1) \sum_{j : \Delta_j > 0} \sqrt{1 + \frac{\log(b_j)}{\Delta_j^2}}\,,
}
where (a) follows from \cref{lem:conc2} and (b) since for all $x, y\geq 0$ it holds that $\sqrt{x+y} \leq \sqrt{x} + \sqrt{y}$ and $\sqrt{xy} \leq x + y$.
To get (c) we bound the first term as in \cref{eq:hard-3}, the second by the fact that arms in $j \in A_2$ have $\Delta_j \leq 2\Delta_J$ and the third
using \cref{lem:tau}.
Finally by substituting this into \cref{eq:hard-1} we have
\eq{
\E\left[\sum_{i \in A_3} \Delta_i \tau_{\Delta_i}\right]
&\leq O_\eta(1) \cdot \left(\E\left[\sum_{j \in A_1} J \sqrt{\tau_j} + \frac{J^2}{\Delta_{\min}}\right] 
    + \sum_{j : \Delta_j > 0} \left(1 + \frac{\log(b_j)}{\Delta_j^2}\right) \right) \\
&\leq O_\eta(1) \sum_{j : \Delta_j > 0} \left(1 + \frac{\log(b_j)}{\Delta_j^2}\right)\,,
}
where the last line follows since $\E[J^2/\Delta_{\min}] = O_\eta(1) \Delta_{\min}^{-1}$ and by Lemma \ref{lem:tau} 
\eq{
\E[J \sqrt{\tau_j}] \leq
\sqrt{\E[J^2] \E[\tau_j]} = O_\eta(1) \cdot \sqrt{1 + \frac{\log(b_j)}{\Delta_j^2}}\,,
}
which completes the proof.

%%%%%%%%%%%%%%%%%%%%%%%%%%%%%%%%%%%%%%%%%%%%%%%%%%%%%%%%%%%%%%%
% TECHNICAL LEMMAS
%%%%%%%%%%%%%%%%%%%%%%%%%%%%%%%%%%%%%%%%%%%%%%%%%%%%%%%%%%%%%%%
\section{Technical Lemmas}\label{app:tech}

\begin{lemma}\label{lem:tech1}
Let $\eta > 1$ and $\rho \in [0,1]$ and $\lambda \in (0, \infty]$ and $x > 0$, then
\eq{
\sum_{k=0}^\infty \min\set{\eta^{k+1},\, \lambda^\rho \eta^{(1-\rho)(k+1)}} \exp\left(-x \eta^k\right) 
&\leq \begin{cases}
\frac{1}{x}\left(\frac{2}{e} + \frac{\eta}{\log(\eta)}\right) & \text{if } x\lambda \geq 1 \\
\frac{\lambda^\rho x^{\rho-1}}{\log(\eta)} \left(1 + \frac{1}{e} + \log\left(\frac{1}{\lambda x}\right)\right) & \text{otherwise}\,.
\end{cases} \\
&= O\left(\frac{\eta}{\eta - 1}\right) \cdot \min\set{\frac{1}{x},\, \lambda^\rho x^{\rho-1} \logp\left(\frac{1}{\lambda x}\right)}\,. 
}
\end{lemma}

\begin{proof}
Let $f(k) = \min\set{\eta^{k+1},\, \lambda^\rho \eta^{(k+1)(1-\rho)}} \exp(-x \eta^k)$, which is unimodal and so
$\sum_{k=0}^\infty f(k) \leq 2 \sup_k f(k) + \int^\infty_0 f(k) dk$.
If $x \lambda \geq 1$, then
\eq{
\int^\infty_0 f(k) dk 
\leq \eta \int^\infty_0 \eta^k \exp\left(-\eta^k x\right) dk 
= \frac{\eta}{x\log(\eta)}\,.
}
If $x\lambda < 1$, then let $k_\lambda$ be such that $\eta^{k_\lambda} = \lambda^\rho \eta^{k_\lambda(1-\rho)}$ and $k_x$ be such that $\eta^k = 1/x$. 
\eq{
\int^\infty_0 f(k) dk 
&\leq \eta \int^{k_\lambda}_0 \eta^k dk 
  + \eta \int^{k_x}_{k_\lambda} \lambda^\rho \eta^{k_x(1-\rho)}  dk
  + \eta \int^\infty_{k_x} \lambda^\rho \eta^{k(1-\rho)} \exp\left(-x \eta^k\right) dk \\
&= \frac{\lambda - 1}{\log(\eta)} + \eta \left(k_x - k_\lambda\right) \lambda^\rho x^{\rho-1} + \eta \lambda^\rho x^{\rho-1} \int^\infty_{k_x} \eta^{(k-k_x)(1-\rho)} \exp\left(-\eta^{k_x - k}\right) dk \\
&\leq \frac{\lambda - 1}{\log(\eta)} + \frac{\eta \lambda^\rho x^{\rho-1} \log\left(\frac{1}{\lambda x}\right)}{\log(\eta)} + \frac{\eta \lambda^\rho x^{\rho-1}}{e \log(\eta)}
}
Finally 
\eq{
\sup_k f(k) \leq \min\set{\frac{1}{e x},\, \eta \lambda^\rho x^{\rho-1}}\,.
}
Therefore
\eq{
\sum_{k=0}^\infty \min\set{\eta^{k+1},\, \lambda^\rho \eta^{(1-\rho)(k+1)}} \exp\left(-x \eta^k\right)
&\leq \begin{cases}
\frac{1}{x}\left(\frac{2}{e} + \frac{\eta}{\log(\eta)}\right) & \text{if } x\lambda \geq 1 \\
\frac{\lambda^\rho x^{\rho-1}}{\log(\eta)} \left(1 + \frac{1}{e} + \log\left(\frac{1}{\lambda x}\right)\right) & \text{otherwise}\,.
\end{cases}
}
\end{proof}

%%%%%%%%%%%%%%%%%%%%%%%%%%%%%%%%%%%%%%%%%%%%%%%%%%%%%
% TABLE OF NOTATION
%%%%%%%%%%%%%%%%%%%%%%%%%%%%%%%%%%%%%%%%%%%%%%%%%%%%%
\section{Table of Notation}\label{app:notation}

\noindent
\renewcommand{\arraystretch}{1.5}
\hspace{-0.3cm}
\begin{tabular}{p{3cm}p{10cm}}
$K$                   & number of arms \\
$n$                   & horizon \\
$t$                   & current time step \\
$\eta$                & constant parameter greater than $1$ determining width of confidence interval \\
$\rho$                & constant parameter in $(1/2,1]$ \\
$\eta_1$, $\eta_2$    & $\eta_1 = (1 + \eta) / 2$ and $\eta_2 = \eta / \eta_1$ \\
$\mu_i$               & expected return of arm $i$ \\
$\hat \mu_{i,s}$      & empirical estimate of return of arm $i$ based on $s$ samples \\
$\hat \mu_i(t)$       & empirical estimate of return of arm $i$ after time step $t$ \\
$\Delta_i$            & gap between the expected returns of the best arm and the $i$th arm \\ 
$\Delta_{\min}$       & minimal non-zero gap $\Delta_{\min} = \min\set{\Delta_i : \Delta_i > 0}$ \\
$\Delta_{\max}$       & maximum gap $\Delta_{\max} = \max_i \Delta_i$ \\
$\logp(x)$            & $\max\set{1, \log(x)}$ \\
$B_i$ and $b_i$           & see \cref{eq:def:bt} \\
$k_{i,\rho}$          & $\sum_{j=1}^K \min\{1, \Delta_i^{2\rho} / \Delta_j^{2\rho}\}$ \\ 
$K_{i,\rho}$          & see \cref{eq:def:K} \\
$\tau_i$     & see \cref{eq:def:bt} \\
$\tau_\Delta$ & see \cref{eq:def:taudelta} \\
$p(\eta)$            & see \cref{lem:lil} \\
$\Opt$                   & set of optimistic arms \cref{def:O} \\
$J$                   & $J = \min \Opt$ \\
\end{tabular}

\end{document}

%%%%%%%%%%%%%%%%%%%%%%%%%%%%%%%%%%%%%%%%%%%%%%%%%%%%%%%%%%%%%%%
%%%%%%%%%%%%%%%%%%%%%%%%%%%%%%%%%%%%%%%%%%%%%%%%%%%%%%%%%%%%%%%
%%%%%%%%%%%%%%%%%%%%%%%%%%%%%%%%%%%%%%%%%%%%%%%%%%%%%%%%%%%%%%%
%%%%%%%%%%%%%%%%%%%%%%%%%%%%%%%%%%%%%%%%%%%%%%%%%%%%%%%%%%%%%%%
%%%%%%%%%%%%%%%%%%%%%%%%%%%%%%%%%%%%%%%%%%%%%%%%%%%%%%%%%%%%%%%
%%%%%%%%%%%%%%%%%%%%%%%%%%%%%%%%%%%%%%%%%%%%%%%%%%%%%%%%%%%%%%%
%%%%%%%%%%%%%%%%%%%%%%%%%%%%%%%%%%%%%%%%%%%%%%%%%%%%%%%%%%%%%%%
%%%%%%%%%%%%%%%%%%%%%%%%%%%%%%%%%%%%%%%%%%%%%%%%%%%%%%%%%%%%%%%
%%%%%%%%%%%%%%%%%%%%%%%%%%%%%%%%%%%%%%%%%%%%%%%%%%%%%%%%%%%%%%%
%%%%%%%%%%%%%%%%%%%%%%%%%%%%%%%%%%%%%%%%%%%%%%%%%%%%%%%%%%%%%%%
%%%%%%%%%%%%%%%%%%%%%%%%%%%%%%%%%%%%%%%%%%%%%%%%%%%%%%%%%%%%%%%
%%%%%%%%%%%%%%%%%%%%%%%%%%%%%%%%%%%%%%%%%%%%%%%%%%%%%%%%%%%%%%%
%%%%%%%%%%%%%%%%%%%%%%%%%%%%%%%%%%%%%%%%%%%%%%%%%%%%%%%%%%%%%%%
%%%%%%%%%%%%%%%%%%%%%%%%%%%%%%%%%%%%%%%%%%%%%%%%%%%%%%%%%%%%%%%